\documentclass[11pt]{article}

\usepackage{amsmath}
\usepackage{amsthm}
\usepackage{amssymb}
\usepackage{amsfonts}
\usepackage{subfigure}
\usepackage{color}
\usepackage{makeidx}
\usepackage{hyperref}

\newcommand{\red}[1]{\textcolor{red}{{#1}}}
\newcommand{\ignore}[1]{}

\newtheorem{theorem}{Theorem}

\newtheorem{lemma}[theorem]{Lemma}

\newtheorem{corollary}[theorem]{Corollary}

\renewcommand{\Pr}{{\bf Pr}}
\newcommand{\E}{{\bf E}}

\begin{document}

\title{On Exact Learning Monotone DNF \\ from Membership
Queries}
\author{Hasan Abasi\ \ \ \ \ \ \ \  Nader H. Bshouty\\ Department of
Computer Science\\ Technion, Haifa, Israel \and Hanna Mazzawi\\ IBM Research - Haifa\\
Haifa, Israel.}
\maketitle

\begin{abstract}
In this paper, we study the problem of learning a monotone DNF with at most $s$ terms
of size (number of variables in each term)
at most $r$ ($s$ term $r$-MDNF) from membership queries. This problem is equivalent to the problem
of learning a general hypergraph using hyperedge-detecting queries, a problem motivated by applications arising in chemical reactions and genome sequencing.

We first present new lower bounds for this problem and then
present deterministic and randomized adaptive algorithms with query complexities
that are almost optimal. All the algorithms we present in this paper run in time linear in
the query complexity and the number of variables $n$.
In addition, all of the algorithms we present in this paper are asymptotically
tight for fixed $r$ and/or $s$.
\end{abstract}

\section{Introduction}
We consider the problem of learning a monotone DNF with at most $s$ terms,
where each monotone term contains at most $r$ variables
($s$ term $r$-MDNF) from membership queries~\cite{A87}.
This is equivalent to the problem of learning a general hypergraph using hyperedge-detecting queries, a problem
that is motivated by applications arising in chemical reaction and genome sequencing.

\subsection{Learning Hypergraph}
A hypergraph is $H=(V,E)$ where $V$ is the set of vertices and $E\subseteq 2^V$
is the set of edges. The dimension of the hypergraph $H$ is the cardinality
of the largest set in $E$. For a set $S\subseteq V$,
the {\it edge-detecting queries} $Q_H(S)$ is answered ``Yes'' or ``No'', indicating whether $S$ contains
all the vertices of at least one edge of $H$.
Our learning problem is equivalent to learning a hidden
hypergraph of dimension $r$ using edge-detecting queries.

This problem has many applications in chemical reactions and genome sequencing.
In chemical reactions, we are given a set of chemicals,
some of which react and some which do not. When multiple chemicals are combined in one
test tube, a reaction is detectable if and only if at least one set of the chemicals in the tube reacts.
The goal is to identify which sets react using as few experiments as possible. The time needed
to compute which experiments to do is a secondary consideration, though it is polynomial for
the algorithms we present \cite{AC08}.
See \cite{GK98,BAK01,ABK04,AA05,AC06,AC08} for more details and other applications.

\subsection{Previous Results}
In \cite{AC08}, Angluin and Chen presented an deterministic optimal adaptive learning algorithm for
learning $s$-term $2$-MDNF. They also gave a lower bound of $\Omega((2s/r)^{r/2}+rs\log n)$
for learning the class of $s$-term $r$-MDNF when $r<s$. In \cite{AC06}, Angluin and Chen
gave a randomized algorithm for $s$-term $r$-uniform MDNF (the size of each term is {\it exactly} $r$)
that asks $O(2^{4r}s\cdot poly(r,\log n))$ membership queries.
For $s$-term $r$-MDNF where $r\le s$, they gave a randomized learning algorithm
that asks $O(2^{r+r^2/2}s^{1+r/2}\cdot poly(\log n))$ membership queries.

Literature has also addressed learning some subclasses of $s$-term 2-MDNF. Those classes have specific applications to genome sequencing.
See \cite{GK98,BAK01,ABK04,AA05,AC06,AC08}. In this paper we are interested
in learning the class of all $s$-term $r$-MDNF formulas for any $r$ and $s$.

\subsection{Our Results}
In this paper, we distinguish between two cases: $s\ge r$ and $s<r$.

For $s<r$, we first prove the lower bound $O((r/s)^{s-1}+rs\log n)$.
We then give three algorithms. Algorithm I is a deterministic algorithm that
asks $O(r^{s-1}+rs\log n)$ membership queries. Algorithm II is a deterministic algorithm that
asks $O(s\cdot N((s-1;r);sr)+rs\log n)$ membership queries where $N((s-1;r);sr)$
is the size of $(sr,(s-1,r))$-cover free family (see Subsection \ref{CVF} for
the definition of cover free) that can be constructed in time linear in its size.
An $(sr,(s-1,r))$-cover free family of size $(r/s)^{s-1+o(1)}$ is known to exist.
For some $r$ and $s$ (for example $r=o(s\log s\log\log s)$),
such a bound can be achieved in linear time and
therefore for those cases, algorithm II is almost optimal. Algorithm III
is a randomized algorithm that asks
$$O\left({s+r\choose s}\sqrt{sr}\log (sr)+rs\log n\right)
=O\left(\left(\frac{r}{s}\right)^{s-1+o(1)}+rs\log n\right)$$ membership queries.
This algorithm is almost optimal.

For the case $s\ge r$, Angluin and Chen, \cite{AC08}, gave
the lower bound $\Omega((2s/r)^{r/2}+rs\log n)$. We give two algorithms
that are almost tight. The first algorithm, Algorithm IV, is a deterministic algorithm
that asks $(crs)^{r/2+1.5}+rs\log n$ membership queries for some constant $c$.
The second algorithm, Algorithm V, is a randomized algorithm that asks
$(c's)^{r/2+0.75}+rs\log n$ membership queries for some constant $c'$.

All the algorithms we present in this paper run in time linear in
the query complexity and $n$.
Additionally, all the algorithms we describe in this paper are asymptotically
tight for fixed $r$ and $s$.

The following table summarizes our results. We have removed the term
$rs\log n$ from all the bounds to be able to fit this table in this page.
Det. and Rand. stands for deterministic algorithm and randomized algorithm, respectively.
\begin{center}
\begin{tabular}{|l|c|l|l|c|}
& Lower Bound& &  Rand./& Upper Bound\\
$r,s$ &$rs\log n+$ & Algorithm& Det.& $rs\log n+$\\
\hline
\hline

$r>s$ & $\left(\frac{r}{s}\right)^{s-1}$ & Alg. I & Det. &$r^{s-1}$\\
\cline{3-5}
&& Alg. II & Det. &$s\cdot N((s-1;r);sr)$\\
\cline{3-5}
&& Alg. III & Rand. & $(\log r)\sqrt{s}e^s\left(\frac{r}{s}+1\right)^s$\\
\hline
$r\le s$ & $\left(\frac{2s}{r}\right)^{r/2}$ & Alg. IV. & Det. &$(3e)^{r} (rs)^{r/2+1.5}$\\
\cline{3-5}
 &  & Alg. IV. & Rand. &$\sqrt{r}(3e)^{r}(\log s) s^{r/2+1}$\\
\hline
\end{tabular}
\end{center}

\section{Definitions and Notations}
For a vector $w$, we denote by $w_i$ the $i$th entry of $w$. For a positive integer~$j$,
we denote by $[j]$ the set $\{1,2,\ldots,j\}$.

Let $f(x_1,x_2,\ldots,x_n)$ be a Boolean
function from $\{0,1\}^n$ to $\{0,1\}$. For an assignment $a\in\{0,1\}^n$
we say that $f$ is $\xi$ in $a$ (or $a$ is $\xi$ in $f$)
if $f(a)=\xi$. We say that $a$ is zero in $x_i$ if $a_i=0$.
For a set of variables $S$, we say that $a$ is zero in
$S$ if for every $x_i\in S$, $a$ is zero in $x_i$. Denote $X_n=\{x_1,\ldots,x_n\}$.

For a Boolean function $f(x_1,\ldots,x_n)$, $1\le i_1<i_2<\cdots<i_k\le n$
and $\sigma_1,\ldots,\sigma_k\in\{0,1\}$ we denote by $$f|_{x_{i_1}=\sigma_1, x_{i_2}=
\sigma_2,\cdots, x_{i_k}=\sigma_k}$$ the function $f$ when fixing the variables
$x_{i_j}$ to~$\sigma_j$ for all $j\in[k]$. We denote by $a|_{x_{i_1}=\sigma_1, x_{i_2}=
\sigma_2,\cdots, x_{i_k}=\sigma_k}$ the assignment $a$ where each $a_{i_j}$ is replaced
by $\sigma_j$ for all $j\in [k]$. Note that
$$f|_{x_{i_1}=\sigma_1, x_{i_2}=
\sigma_2,\cdots, x_{i_k}=\sigma_k}(a)=f(a|_{x_{i_1}=\sigma_1, x_{i_2}=
\sigma_2,\cdots, x_{i_k}=\sigma_k}).$$
When $\sigma_1=\cdots=\sigma_k=\xi$ and $S=\{x_{i_1},\ldots,x_{i_k}\}$, we denote $$f|_{x_{i_1}=\sigma_1, x_{i_2}=
\sigma_2,\cdots, x_{i_k}=\sigma_k}$$ by
$f|_{S\gets \xi}$. In the same way, we define $a|_{S\gets \xi}$.
We denote by $1^n=(1,1,\ldots,1)\in\{0,1\}^n$.

For two assignments $a,b\in \{0,1\}^n$, we write $a\le b$ if
for every $i$, $a_i\le b_i$. A Boolean function $f:\{0,1\}^n\to\{0,1\}$ is {\it monotone} if
for every two assignments $a,b\in\{0,1\}^n$, if $a\le b$ then $f(a)\le f(b)$.
Recall that every monotone Boolean function $f$ has a unique representation
as a reduced monotone DNF. That is, $f = M_1\vee M_2 \vee \cdots \vee M_s$ where
each {\it monomial} $M_i$ is an ANDs of input variables, and for every
monomial $M_i$ there is a unique assignment $a^{(i)}$ such that $f(a^{(i)})=1$
and for every $j\in [n]$ where $a^{(i)}_j=1$ we have $f(a^{(i)}|_{x_j= 0})=0$. We call
such assignment a {\it minterm} of the function $f$. Notice that
every monotone DNF can be uniquely determined by its minterms.

For a monotone DNF, $f(x_1,x_2,\ldots,x_n)
= M_1\vee M_2 \vee \cdots \vee M_s$, and a variable~$x_i$, we say that $x_i$ is
$t$-\textit{frequent} if it appears in more than or equal to $t$ terms. A monotone DNF $f$ is
called {\it read $k$ monotone} DNF, if none of its variables is $k+1$-frequent.

\subsection{Learning Model}
Consider a {\it teacher} (or a black box) that has a {\it target function}
$f:\{0,1\}^n\to \{0,1\}$ that is $s$-term $r$-MDNF. The teacher
can answer {\it membership queries}.
That is, when receiving $a\in\{0,1\}^n$ it returns $f(a)$.
A {\it learning algorithm} is an algorithm
that can ask the teacher membership queries.
The goal of the learning algorithm is to
{\it exactly learn} (exactly find) $f$ with minimum number of
membership queries and optimal time complexity.

In our algorithms, for a function $f$ we will denote by $MQ_f$ the oracle that
answers the membership queries. That is, for $a\in \{0,1\}^n$,
$MQ_f(a)=f(a)$.

\subsection{Cover-Free Families}\label{CVF}
\ignore{A $d$-{\it restriction problem} \cite{NSS95,AMS06,B12} is a
problem of the following form: Given $\Sigma=\{0,1\}$, a length
$n$ and a set $B\subseteq \Sigma^d$ of assignments. Find a set
$A\subseteq \Sigma^n$ of small size such that: For any $1\le i_1<
i_2<\cdots < i_d\le n$ and $b\in B$ there is $a\in A$ such that
$(a_{i_1},\ldots,a_{i_d})=b$.

When $B=\{0,1\}^d$ then $A$ is called $(n,d)$-{\it universal set}.
The lower bound for the size of $(n,d)$-universal set is
$|A|=\Omega(2^d\log n)$,~\cite{KS72}. The union bound gives the
upper bound $O(d2^d\log n)$ and one can show that a random uniform
set of $O(d2^d\log n)$
assignments in $\{0,1\}^n$ is $(n,d)$-universal set with high
probability. The best known polynomial time construction
(poly$(2^d,n)$) for this problem gives a universal set of size
$2^{d+O(\log^2d)}\log n$~\cite{NSS95}.}

The problem $(n,(s, r))$-{\it cover-free family} \cite{KS64}
is equivalent to the following problem: A $(n,(s,r))$-cover-free
family is a set $A\subseteq \{0,1\}^n$ such that for every $1\le
i_1< i_2<\cdots < i_d\le n$ where $d = s +r$ and every $J
\subseteq [d]$ of size $|J|=s$ there is $a\in A$ such that
$a_{i_k} = 0$ for all $k \in J$ and $a_{i_j} = 1$ for all $j
\not\in J$. Denote by $N((s; r); n)$ the minimum size of such set.
The lower bounds in
\cite{SWZ00} are
$$N((s;r);n)\ge \Omega\left(\frac{(s+r)}
{\log{s+r\choose s}}{s+r\choose s}\log n\right).$$
It is known that a set of random
\begin{eqnarray}\label{rand}
m=O\left(\sqrt{{\min(r,s)}}{s+r\choose s}\left ((s+r)\log n+\log\frac{1}{\delta}\right)\right)
\end{eqnarray}
vectors $a^{(i)}\in \{0,1\}^n$, where each $a^{(i)}_j$ is $1$ with probability $r/(s+r)$, is a $(n,(s,r))$-cover free
family with probability at least $1-\delta$.

In \cite{B12}, Bshouty gave
a deterministic construction of $(n,(s,r))$-CFF
of size
\begin{eqnarray}\label{Bsh}C&:=&\min((2e)^s r^{s+3},(2e)^r s^{r+3})\log n\nonumber \\
&=&{s+r\choose r} 2^{\min(s\log s,r\log r)(1+o(1))}\log n\end{eqnarray}  that can be
constructed in time $C\cdot n$. Fomin et. al. in \cite{FLS14} gave
a construction of size
\begin{eqnarray}\label{Fom}D:={s+r\choose r} 2^{O\left(\frac{r+s}{\log\log (r+s)}\right)}\log n\end{eqnarray}
that can be constructed in time $D\cdot n$.
The former bound, (\ref{Bsh}), is better than the latter when $s\ge r\log r\log\log r$ or $r\ge s\log s\log\log s$.
We also note that the former bound, (\ref{Bsh}),
is almost optimal, i.e., $${s+r\choose r}^{1+o(1)}\log n,$$ when
$r=s^{\omega(1)}$ or $r=s^{o(1)}$ and the latter bound, (\ref{Fom}),
is almost optimal when
$$o(s\log\log s\log\log\log s)=r=\omega\left(\frac{s}{\log\log s\log\log\log s}\right).$$

\section{Lower Bounds}
In this section, we prove some lower bounds.

\subsection{General Lower Bound}
In this section, we prove that the information theoretic
lower bound for learning a class $C$ from membership queries
is also a lower bound for any randomized
learning algorithm. We believe it is
a folklore result, but we could not find the proof in the literature.
We first state the following information-theoretic lower bound for
deterministic learning algorithm,

\begin{lemma} Let $C$ be any class of Boolean function. Then any
deterministic
learning algorithm for $C$ must ask at least $\log |C|$ membership queries.
\end{lemma}

We now prove,
\begin{lemma} Let $C$ be any class of boolean function. Then any
Monte Carlo (and therefore, Las Vegas) randomized
learning algorithm that learns $C$ with probability
at least $3/4$ must ask at least $\log |C| -1$ membership queries.
\end{lemma}
\begin{proof} Let ${\cal A}$ be a randomized algorithm that
for every $f\in C$ and
an oracle $MQ_f$ that answers membership queries for $f$, asks $m$ membership queries
and satisfies
$$\Pr_{s}[{\cal A}(MQ_f,s)=f]\ge \frac{3}{4}$$
where $s\in \{0,1\}^N$ is chosen randomly uniformly for some large $N$.
Consider the random variable $X_f(s)$ that is $1$ if ${\cal A}(MQ_f,s)=f$
and $0$, otherwise. Then for every $f$, $\E_s[X_f]\ge 3/4$. Therefore,
for random uniform $f\in C$
$$3/4\le \E_f[\E_s[X_f]] =\E_s[\E_f[X_f(s)]].$$ and by Markov Bound
for at least $1/2$ of the elements
$s\in \{0,1\}^N$ we have $\E_{f}[X_f(s)]\ge 1/2$. Let $S\subseteq \{0,1\}^N$
be the set of such elements. Then $|S|\ge 2^N/2$. Let $s_0\in S$ and $C_{s_0}\subseteq C$ the class
of functions $f$ where $X_{f}(s_0)=1$. Then $|C_{s_0}|\ge |C|/2$
and ${\cal A}(MQ_f,s_0)$ is a deterministic
algorithm that learns the class $C_{s_0}$. Using the information theoretic
lower bound for deterministic algorithm, we conclude that ${\cal A}(MQ_f,s_0)$
must ask at least
$$m\ge \log |C_{s_0}| = \log(1/2)+\log |C|$$ membership queries.
\end{proof}

Specifically, we have,
\begin{corollary} Any
Monte Carlo (and therefore Las Vegas) randomized
learning algorithm for the class of $s$-term $r$-MDNF must ask on
average at least $rs\log n$ membership queries.
\end{corollary}

\subsection{Two Lower Bounds}
In this section, we give two lower bounds. The first is from \cite{AC06}
and the second follows using the same techniques used in \cite{BGHS96}.

In \cite{AC06}, Angluin and Chen proved,
\begin{theorem}  Let $r$ and $s$ be integers. Let $k$ and $\ell$ be two
integers such that
$$\ell\le r, \ s\ge {k\choose 2}\ell+1.$$
Any (Monte Carlo) randomized
learning algorithm for the class of $s$-term $r$-MDNF must ask at least
$$k^\ell-1$$ membership queries.

Specifically, when $s>>r$ we have the lower bound
$$\Omega\left(\left(\frac{2s}{r}\right)^{r/2}\right)$$
membership queries. Also, for any integer $\lambda$ where
$${\lambda\choose 2}r+1\le s < {\lambda+1\choose 2}r$$ we have the lower bound $\lambda^r-1$.
\end{theorem}

We now prove the following lower bound,
\begin{theorem}  Let $r$ and $s$ be integers and $\ell$ and $t$ be two
integers such that
$$\ell-\left\lfloor\frac{\ell}{t}\right\rfloor\le r,\ \ \ \left\lfloor\frac{\ell}{t}\right\rfloor\le s-1.$$
Any
(Monte Carlo) randomized
learning algorithm for the class of $s$-term $r$-MDNF must ask at least $t^{\lfloor \ell/t\rfloor}$
membership queries.

Specifically, for $r>>s$ we have the lower bound
$$\left(\frac{r}{s}\right)^{s-1}.$$
and for any constant integer $\lambda$ and $\lambda s\le r<(\lambda+1)s$ we have the lower bound
$$(\lambda+1)^{s-1}.$$
\end{theorem}
\begin{proof}
Let $m= \lfloor\ell/t\rfloor$. Consider the monotone terms $M_j=x_{(j-1)t+1}\cdots x_{jt}$
for $j=1,2,\ldots, m$. Define $M_{i,k}$ where
$i=1,\ldots,m$ and $k=1,\ldots,t$ the monotone term
$M_i$ without the variable $x_{(i-1)t+k}$.
Let $M_{k_1,k_2,\ldots,k_m}=M_{1,k_1}M_{2,k_2}\cdots M_{m,k_m}$.
The only way we can distinguish between the two hypothesis $f=M_1\vee M_2\vee\cdots\vee M_m$
and $g=M_1\vee M_2\vee\cdots\vee M_m\vee M_{k_1,k_2,\ldots,k_m}$ is by guessing
an assignment that is $1$ in all its first $mt$ entries except for the
entire $k_1,t+k_2,2t+k_3,\ldots,(m-1)t+k_m$. That is, by guessing $k_1,k_2,\ldots,k_m$.
This takes an average of $t^m$ guesses.  Since both $f$ and $g$ are $s$-term $r$-MDNF, the result follows.

For $r>>s$, we choose $\ell=r$ and $t$ such that $\lfloor \ell/t\rfloor=s-1$.
Since $s-1=\lfloor \ell/t\rfloor\ge \ell/t-1$, we have $t\ge r/s$ and the result follows.

For $\lambda s\le r<(\lambda+1)s$, proving the lower bound for $r=\lambda s$ is sufficient.
Take $t=\lambda+1$ and $\ell=(\lambda+1)s-1$.
\end{proof}

\section{Optimal Algorithms for Monotone DNF}
In this section, we present the algorithms (Algorithm I-V) that learn
the class of $s$-term $r$-MDNF. We first give a simple algorithm that
learns one term. We then give three algorithms (Algorithm I-III) for the case $r>s$
and two algorithms (Algorithm IV-V) for the case $s\ge r$.

\subsection{Learning One Monotone Term}\label{LOMT}
In this section, we prove the following result.
\begin{lemma}
Let $f(x)=M_1\vee M_2\vee \cdots \vee M_{s}$ be the target function where each
$M_i$ is a monotone term of size at most $r$. Suppose $g(x)=M_1\vee M_2\vee \cdots \vee M_{s'}$
and $h(x)=M_{s'+1}\vee M_{s'+2}\vee \cdots \vee M_s$. If $a$ is an assignment such that
$g(a)=0$ and $h(a)=1$, then a monotone term in $h(x)$ can be found with $$O\left(r\log\frac{n}{r}\right)$$ membership queries.
\end{lemma}
\begin{proof} First notice that since $g$ is monotone,
for any $b\le a$ we have $g(b)=0$.
Our algorithm finds a minterm $b\le a$ of $f$ and therefore $b$ is a minterm
of $h$.

First, if the number of ones in $a$ is $2r$, then we can find a minterm
by flipping each bit in $a$
that does change the value of $f$ and get a minterm. This takes at most $2r$ membership queries.

If the number of ones in $a$ is $w>2r$, then we divide the entries of $a$ that
are equal to $1$ into $2r$ disjoint sets $S_1,S_2,\ldots,S_{2r}$ where for every $i$,
the size of $S_i$ is either
$\lfloor w/(2r)\rfloor$ or $\lceil w/(2r)\rceil$. Now for $i=1,2,\ldots,2r$,
we flip all the entries of $S_i$ in $a$ to zero and ask a membership query. If
the function is one, we keep those entries $0$. Otherwise we set them back to $1$
and proceed to $i+1$. At the end of this procedure, at most $r$ sets are not flipped.
Therefore, at least half of the bits in $a$ are flipped to zero using $2r$ membership queries.
Therefore, the number of membership queries we need to get a minterm is
$2r\log(n/2r)+2r.$
\end{proof}

We will call the above procedure {\bf Find-Term}.

\subsection{The case $r> s$}

In this section, we present three algorithms, two deterministic and one randomized.
We start with the deterministic algorithm.

\subsubsection{Deterministic Algorithm}
Consider the class $s$-term $r$-MDNF.
Let $f$ be the target function. Given $s-\ell$ monotone terms
$M_1\vee M_2\vee \cdots \vee M_{s-\ell}$ that are known to the learning
algorithm to be in $f$.
The learning algorithm goal is to find a new monotone term.
In order to find a new term we need to find an
assignment $a$ that is zero in $M_1\vee M_2\vee \cdots \vee
M_{s-\ell}$ and $1$ in the function $f$. Then by the procedure
{\bf Find-Term} in Subsection~\ref{LOMT}, we get a new term in $O(r\log n)$ additional membership queries.

To find such an assignment, we present three algorithms:

{\bf Algorithm I}:  (Exhaustive Search) choose a variable from each $M_i$ and set it to zero
and set all the other variables to $1$. The set of all such assignments
is denoted by $A$. If $f$ is $1$ in some $a\in A$, then find
a new term using {\bf Find-Term}.

We now show,
\begin{lemma} If $f\not\equiv h$, then {\bf Algorithm I} finds a new term in $r^{s-\ell}+O(r\log n)$
membership queries.
\end{lemma}
\begin{proof}
Since the number of variables in each term in $h:=M_1\vee M_2\vee \cdots \vee
M_{s-\ell}$ is at most $r$ the number of assignments in $A$ is at most $r^{s-\ell}$.
Since we choose one variable from each term in $h$ and set it to zero, all the
assignments in $A$ are zero in $h$. We now show that one of the assignments
in $A$ must be $1$ in $f$, and therefore a new term can be found.

Let $b$ be an assignment that is $1$ in $f$ and zero in $h$. Such assignment
exists because otherwise $f\Rightarrow h$ and since $h\Rightarrow f$ we get $f\equiv h$.
Since $h(b)=0$ there is at least one variable $x_{j_i}$ in each $M_i$ that is zero in $b$.
Then the assignment $a:=1^n|_{x_{j_1}= 0,\ldots,x_{j_{s-\ell}}= 0}$ is in $A$
and $h(a)=0$. Since $a\ge b$ we also have $f(a)=1$.
\end{proof}

The number of queries in this algorithm is
$$\sum_{\ell=1}^{s} O\left(r^{s-\ell}+r\log n\right)=O(r^{s-1}+rs\log n).$$

We now present the second algorithm. Recall that $X_n=\{x_1,\ldots,x_n\}$.

\begin{figure}[h!]
  \begin{center}
  \fbox{\fbox{\begin{minipage}{28em}
  \begin{tabbing}
  xxxx\=xxxx\=xxxx\=xxxx\= \kill
  {\bf Algorithm II}\\
  1) Let $V$ be the set of variables that appear in $M_1\vee M_2\vee \cdots
\vee M_{s-\ell}$.\\
  2) Take a $(|V|,(s-\ell,r))$-CFF $A$ over the variables
$V$.\\
  3) For each $a\in A$\\
  \> 3.1) Define an assignment $a'$
that is $a_i$ in $x_i$ for every $x_i\in V$ \\ \>\> and $1$ in $x_i$ for every
$x_i\in X_n \backslash V$.\\
  \> 3.2) If $M_1\vee M_2\vee \cdots \vee
M_{s-\ell}$ is $0$ in $a'$ and $f$ is one in $a'$ \\
  \>\> then find a new term using {\bf Find-Term}
  \end{tabbing}
  \end{minipage}}}
  \end{center}
	\caption{Algorithm II for the case $r> s$.}
	\label{AlgII}
	\end{figure}

We now show,
\begin{lemma} If $f\not\equiv h$, then {\bf Algorithm II} finds a new term in $N((s-\ell,r),(s-\ell)r)+O(r\log n)$
membership queries.
\end{lemma}
\begin{proof} Let $h:=M_1\vee M_2\vee \cdots \vee
M_{s-\ell}$. Let $b$ be an assignment that is $1$ in $f$ and zero in $h$.
Since $h(b)=0$, there is at least one variable $x_{j_i}$ in each $M_i$ that is zero in $b$.
Consider the set $U=\{x_{j_i}|i=1,\ldots,s-\ell\}$.
Since $f(b)=1$ there is a new term $M$ in $f$ that is one in $b$. That is, all of its variables are
one in $b$. Let $W$ be the set of all variables in $M$. Since $A$ is $(|V|,(s-\ell,r))$-CFF
and since $|U\cup (W\cap V)|\le s-\ell+r$ there is an assignment $a\in A$ that
is $0$ in each variable in $U$ and is one in each variable in $W\cap V$.
Since $a'$ is also $0$, in each variable in $U$ we have $h(a')=0$. Since $a'$ is one in
each variable in $W\cap V$
and one in each variable $W\backslash V$, we have $M(a')=1$ and therefore $f(a')=1$. This completes the proof.
\end{proof}

The number of queries in Algorithm II
is
$$\sum_{\ell=1}^{s-1}N((s-\ell,r),(s-\ell)r)+r\log n=O( sN((s-1,r),sr)+rs\log n).$$

\subsubsection{Randomized Algorithm}
Our third algorithm, Algorithm III, is a randomized algorithm.
It is basically Algorithm II where an $(rs,(s-1,r))$-CFF $A$ is
randomly constructed, as in (\ref{rand}). Notice that
an $(rs,(s-1,r))$-CFF is also an $(|V|,(s-\ell,r))$-CFF, so it can be used
in every round of the algorithm.
The algorithm fails if
there is a new term that has not been found and this happens if and
only if $A$ is not $(rs,(s-1,r))$-CFF. So the failure probability is $\delta$.
By (\ref{rand}), this gives a Monte Carlo randomized algorithm with query complexity
$$O\left(\sqrt{s}{s+r\choose s}\left (r\log r+\log\frac{1}{\delta}\right)+rs\log n\right).$$

\subsection{The case $r<s$}
In this section, we present two algorithms. Algorithm IV
is deterministic and Algorithm V is randomized. We start with the deterministic algorithm.
\subsubsection{Deterministic Algorithm}

In this section, we present Algorithm IV, used when $r<s$.
For this case, we prove the following,
\begin{theorem}\label{s-dominant}
There is a deterministic learning algorithm for
the class of $s$-term $r$-MDNF that asks
$$
O\left((3e)^{r} (rs)^{r/2+1.5} + rs\log n\right),
$$
membership queries.
\end{theorem}

Before proving this theorem, we first prove learnability in simpler settings.
We prove the following,

\begin{lemma}\label{restricted}
Let $f(x_1,x_2,\ldots,x_n)=M_1\vee \cdots\vee M_s$ be the target $s$-term $r$-MDNF.
Suppose the learning algorithm
knows some of the terms, $h = M_1\vee M_2\vee \cdots \vee M_{s-\ell}$ and knows
that $M_{s-\ell+1}$ is of size $r'$.
Suppose that~$h$ is a read $k$ monotone DNF.
Then, there exists an algorithm that finds a new term
(not necessarily $M_{s-\ell+1}$) using
$$
O\left(N((r'k;r');sr)) + r\log n\right),
$$
membership queries.
\end{lemma}
\begin{proof} Consider the algorithm in Figure~\ref{changingC23}.

  \begin{figure}[h!]
  \begin{center}
  \fbox{\fbox{\begin{minipage}{28em}
  \begin{tabbing}
  xxxx\=xxxx\=xxxx\=xxxx\= \kill
  {\bf LearnRead($MQ_f,s,\ell,r'$)}\\
  1) Let $V$ be the set of variables that appear in $h$. \\
  2) Let $A$ be a $(|V|,(r'k,r'))$-CFF
over the variables $V$.\\
  3) For each $a\in A$\\
  \> 3.1) Let $a'\in \{0,1\}^n$ where $a'$ is $a_i$ in each
 $x_i\in V$, \\
  \> \> and one in each $x_i\in X_n\backslash V$.\\
  \> 3.2) $X\gets \O$.\\
  \> 3.3) For each $M_i$, $i=1,\ldots,s-\ell$ such that $M_i(a')=1$ do\\
  \> \> \> Take any variable $x_j$ in $M_i$ and set $X\gets X\cup \{x_j\}$\\
  \> 3.4) Set $a''\gets a'|_{X\gets 0}$.\\
  \> 3.5) If $f(a'')=1$ and $h(a'')=0$ then find a new term using {\bf Find-Term}.
  \end{tabbing}
  \end{minipage}}}
  \end{center}
	\caption{Finding a new term in read $k$.}
	\label{changingC23}
	\end{figure}

Let $V$ be the set of variables that appear in $h$.
Let $M:=M_{s-\ell+1}$. Let $U$ be the set of variables in
$M$ and $W=U\cap V$. Each variable in $W$ can appear in at most
$k$ terms in $h$. Let w.l.o.g $h':=M_1\vee \cdots\vee M_{t}$ be those
terms. Notice that $t\le |W|k\le r'k$. In each term $M_i$, $i\le t$
one can choose a variable $x_{j_i}$ that is not in $W$. This is because,
if all the variable in $M_i$ are in $W$, then $M\Rightarrow M_i$ and
then $f$ is not reduced MDNF.

Let $Z=\{x_{j_i}|i=1,\ldots,t\}$.
Since $|Z|\le t\le r'k$ and $|U|\le r'$ there is $a\in A$ that
is $0$ in every variable in $Z$ and is $1$ in every variable in $U$.
Now notice that $a'$ in step 3.1 in the algorithm
is the same as $a$ over the variables in $Z$ and therefore
$h'(a')=0$. Also $a'$ is the same as $a$ over the variables in $U$ and therefore
$M(a')=1$. Now notice that since $M_i(a')=0$ for $i\le t$, in step 3.4 in the algorithm
we only flip $a'_i$ that correspond to variables in the terms $M_i$, $i>t$.
The set of variables in each other term $M_i$, $i>t$ is disjoint with $U$.
Therefore if for some $i>t$, $M_i(a')=1$ then setting any variable $x_j$ in $M_i$ that is one in $a'$ to zero
will not change the values $M(a')=1$ and (from monotonicity) $h'(a')=0$.
Eventually, we will have an assignment $a''$ that satisfies $h(a'')=0$ and $M(a'')=1$ which implies $f(a'')=1$.
\end{proof}

In the following lemma, we remove the restriction on $h$.
\begin{lemma}\label{no-rest}
Let $f(x_1,x_2,\ldots,x_n)=M_1\vee \cdots\vee M_s$ be the target $s$-term $r$-MDNF.
Suppose some of the terms, $h = M_1\vee M_2\vee \ldots \vee M_{s-\ell}$, are already known
to the learning algorithm.
Then, for any integer $d$, there exists an algorithm that finds a new term using
$$
O\left(\sum_{i=1}^r \binom{r\sqrt{ds}}{i} N(((r-i)\sqrt{s/d}; (r-i));rs) + r\log n\right),
$$
membership queries.
\end{lemma}
\begin{proof} Consider the algorithm in Figure~\ref{changingC44}.

  \begin{figure}[h!]
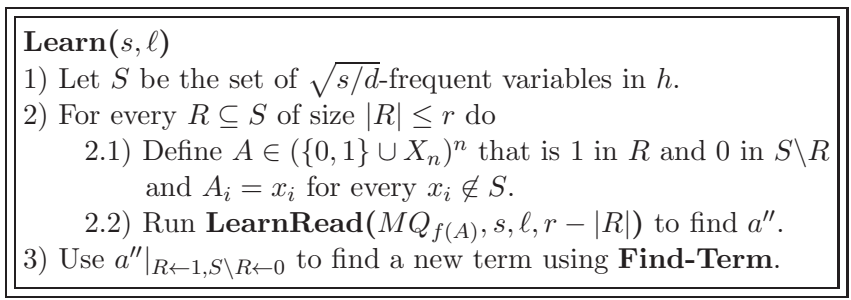

  \begin{center}
  \fbox{\fbox{\begin{minipage}{28em}
  \begin{tabbing}
  xxxx\=xxxx\=xxxx\=xxxx\= \kill
  {\bf Learn($s,\ell$)}\\
  1) Let $S$ be the set of $\sqrt{s/d}$-frequent variables in $h$. \\
  2) For every $R\subseteq S$ of size $|R|\le r$ do\\
   \> 2.1) Define $A\in (\{0,1\}\cup X_n)^n$ that is $1$ in $R$ and 0 in $S\backslash R$\\
           \>\> and $A_i=x_i$ for every $x_i\not\in S$.\\
   \> 2.2) Run {\bf LearnRead($MQ_{f(A)},s,\ell,r-|R|$)} to find $a''$.\\
  3) Use $a''|_{R\gets 1,S\backslash R\gets 0}$ to find a new term using {\bf Find-Term}.
  \end{tabbing}
  \end{minipage}}}
  \end{center}
	\caption{Finding a new term.}
	\label{changingC44}
	\end{figure}

First note that in step 2.2, $f(A)$ is considered in {\bf LearnRead}
as a function in all the variables $X_n$. Note also that the oracle $MQ_{f(A)}$ can
be simulated by $MQ_{f}$, since $f(A)(a)=f(a|_{R\gets 0,S\backslash R\gets 1})$.

Let $W$ be the set of variables that appear in $M:=M_{s-\ell+1}$
and $R=S \cap W$. Note that $A$ is zero in all $S\backslash R$ and $1$ in $R$
and therefore $f(A)$ is now a read $\sqrt{s/d}$ and $M(A)$ contains at most $|W\backslash R|\le r-|R|$
variables. Therefore, when we run {\bf LearnRead}($MQ_{f(A)},s,\ell,r-|R|$) we find an assignment $a''$
that is $1$ in $M(A)$ and zero in $f(A)$ and then $a''|_{R\gets 0,S\backslash R\gets 1}$ is one in $f$
and zero in~$h$.

We now find the number of queries.
By the Pigeon hole principle, there are at most $|S|\le r\sqrt{ds}$ that
are $\sqrt{s/d}$-frequent. The number of sets $R\subseteq S$ of size $i$
is  ${r\sqrt{ds}\choose i}$. For each set, we run {\bf LearnRead($MQ_{f(A)},s,\ell,r-|R|$)}
that by Lemma~\ref{restricted} asks $N(((r-i)\sqrt{s/d}; (r-i));rs)$ queries.
This implies the result.
\end{proof}

We now prove our main result. We choose $d=r$. Then by the construction
(\ref{Bsh}), we have
\begin{eqnarray*}
\binom{r\sqrt{ds}}{i} N(((r-i)\sqrt{s/d}; (r-i));rs)&\le & \left(\frac{er\sqrt{rs}}{i}\right)^i (2e)^{r-i} \left(\frac{(r-i)\sqrt{s}}{\sqrt{r}}\right)^{r-i+3}\\
&\le & e^r 2^{r-i}(\sqrt{rs})^{r+3} \left(\frac{r}{i}\right)^i \left(\frac{r-i}{r}\right)^{r-i+3}\\
&\le &  e^r 2^{r-i}{r\choose i} (\sqrt{rs})^{r+3}.
\end{eqnarray*}

and therefore
$$\sum_{i=1}^r \binom{r\sqrt{ds}}{i} N(((r-i)\sqrt{s/d}; (r-i));rs)\le (3e)^r (rs)^{r/2+1.5}.$$

\subsubsection{Randomized Algorithm}
In this section, we give a randomized algorithm for the case $s> r$.

The randomized algorithm is the same as the deterministic one, except that
each CFF is constructed randomly, as in (\ref{rand}) with probability of success $1-\delta/s$.
We choose $d=1$ and get
\begin{eqnarray*}
\binom{r\sqrt{ds}}{i} N(((r-i)\sqrt{s/d}; (r-i));rs) \ \ \ \ \ \ \ \ \ \ \ \ \ \ \ \ \ \ \ \ \ \ \ \ \ \ \\
\le  \left(\frac{er\sqrt{s}}{i}\right)^i \sqrt{r}(e(\sqrt{s}+1))^{r-i}\left(2s\log rs+\log\frac{s}{\delta}\right).\\
\le e^r2^{r-i}  \left(\frac{r}{i}\right)^i \sqrt{r}s^{r/2}(s\log s+\log(1/\delta))
\end{eqnarray*}

and therefore
$$\sum_{i=1}^r \binom{r\sqrt{ds}}{i} N(((r-i)\sqrt{s/d}; (r-i));rs)\le \sqrt{r}(3e)^r s^{r/2}(s\log s+\log(1/\delta)).$$

\section{Conclusion and Open Problems}
In this paper, we gave an almost optimal adaptive exact learning algorithms for the class of $s$-term
$r$-MDNF. When $r$ and $s$ are fixed, the bounds are asymptotically tight.
Some gaps occur between the lower bounds and upper bounds.
For $r\ge s$, the gap is $c^s$ for some constant $c$ and for
$r\le s$ the gap is $r^{r/2}$.
It is interesting to close these gaps. Finding a better
deterministic construction of CFF will give
better deterministic algorithms.

Another challenging problem is finding tight bounds for non-adaptive learning of this class.

\end{document}